\documentclass[letterpaper]{article} %
\usepackage{aaai25}  %
\usepackage{times}  %
\usepackage{helvet}  %
\usepackage{courier}  %
\usepackage[hyphens]{url}  %
\usepackage{graphicx} %
\usepackage{natbib}  %
\usepackage{caption} %
\usepackage{algorithm}
\usepackage{algorithmic}

\usepackage{newfloat}
\usepackage{listings}
\DeclareCaptionStyle{ruled}{labelfont=normalfont,labelsep=colon,strut=off} %
\floatstyle{ruled}
\newfloat{listing}{tb}{lst}{}
\floatname{listing}{Listing}
\usepackage{amsfonts}       %
\usepackage{nicefrac}       %
\usepackage{amsmath}
\usepackage{amsthm}
\usepackage{mathtools}
\usepackage{caption}
\usepackage{subcaption}

\title{A Wiener Process Perspective on Local Intrinsic Dimension Estimation Methods}

\author{
    Piotr Tempczyk\equalcontrib\textsuperscript{\rm 1,2,3},
    Łukasz Garncarek\textsuperscript{\rm 3,4},
    Dominik Filipiak\textsuperscript{\rm 5,6,7},
    Adam Kurpisz\equalcontrib\textsuperscript{\rm 3,8,9}
}
\affiliations{
    \textsuperscript{\rm 1}Institute of Informatics, University of Warsaw, Poland\\
    \textsuperscript{\rm 2}NASK National Research Institute, Warsaw, Poland\\
    \textsuperscript{\rm 3}PL4AI, Poland\\
    \textsuperscript{\rm 4}Snowflake, USA\\
    \textsuperscript{\rm 5}Adam Mickiewicz University, Pozna\'n, Poland\\
    \textsuperscript{\rm 6}Perelyn, Warsaw, Poland\\
    \textsuperscript{\rm 7}University of Innsbruck, Austria\\
    \textsuperscript{\rm 8}BFH Bern Business School, Switzerland\\
    \textsuperscript{\rm 9}ETH Zurich, Switzerland\\
    tempczyk.piotr@gmail.com, lukgar@gmail.com, df@amu.edu.pl, adam.kurpisz@ifor.math.ethz.ch
}

\begin{document}

\newtheorem{theorem}{Theorem}[section]
\newtheorem{proposition}[theorem]{Proposition}
\newtheorem{fact}[theorem]{Fact}
\newtheorem{corollary}[theorem]{Corollary}
\newtheorem{lemma}[theorem]{Lemma}
\newtheorem*{FSLD}{Fick's Second Law of Diffusion}

\theoremstyle{definition}
\newtheorem{definition}[theorem]{Definition}

\newcommand{\NN}[2]{\mathcal{N}\left(#1, #2\right)}
\newcommand{\RR}{\mathbb{R}}
\newcommand{\one}{\mathbb{1}}
\newcommand{\norm}[1]{\left\lVert{#1}\right\rVert}

\makeatletter
\newcommand{\dummylabel}[2]{\def\@currentlabel{#2}\label{#1}}
\makeatother
\dummylabel{app:laplacian-diffused-density}{B}
\dummylabel{lem:app:laplacian-rho-t-everywhere}{B.1}
\dummylabel{cor:laplacian-rho-t-on-manifold}{B.2}
\dummylabel{app:equivalent-formulations-of-lidl}{C}
\dummylabel{prop:app:equivalent-asymptotics}{C.1}
\dummylabel{prop:app:equivalent-lidl}{C.2}
\dummylabel{prop:app:beta_final}{C.3}
\dummylabel{lem:convex-combination}{D.1}
\dummylabel{lem:convex-combination-coeff-estimate}{D.2}
\dummylabel{eq:convex_comb_of_beta}{19}

\maketitle

\begin{abstract}
    Local intrinsic dimension (LID) estimation methods have received a lot of attention in recent years thanks to the progress in deep neural networks and generative modeling. In opposition to old non-parametric methods, new methods use generative models to approximate diffused dataset density to scale the methods to high-dimensional datasets (e.g. images). In this paper, we investigate the recent state-of-the-art parametric LID estimation methods from the perspective of the Wiener process. We explore how these methods behave when their assumptions are not met. We give an extended mathematical description of those methods and their error as a function of the probability density of the data.
\end{abstract}

\section{Introduction}

LID estimation has gained increasing attention in recent years as part of the fast-growing field of topological data analysis.
It was able to progress from non-parametric to parametric methods thanks to the latest progress in the field of generative modeling.
LID estimation methods are algorithms for estimating the manifold's dimension from which the data point $x$ was sampled.
In these methods, we assume that the data lie on the union of one or more manifolds that may be of different dimensions.

The estimation of intrinsic dimension is substantial for data analysis and machine learning \citep{ansuini2019intrinsic,li2018measuring,rubenstein2018latent} and was investigated in relation to dimensionality reduction and clustering \citep{vapnik2013nature,kleindessner2015dimensionality,camastra2016intrinsic}, analyzing the training and representation learning processes within deep neural networks \citep{li2018measuring,ansuini2019intrinsic,pope2020intrinsic,loaiza2024deep}, verifying the union of manifolds hypothesis for images \cite{brown2022verifying}, and used to improve out-of-distribution detection algorithm \cite{pmlr-v235-kamkari24a}.

Prior to introducing LIDL \cite{tempczyk2022lidl}, two approaches to solve the intrinsic dimension estimation problem were presented.
The first employs global methods; see, e.g.,~\citet{Fukunaga1971,minka2000,fan2010}.
These methods are known to suffer from issues related to a manifold curvature and non-uniformity of the data distribution.
They also assume that data lies on a single manifold of constant dimension, so dimensionality is the same for all $x$.
The second approach is based on local non-parametric methods that explore the geometric properties of neighborhoods~\cite{johnsson2014low,levina2004maximum} calculating some statistics using points from the neighborhood of $x$.
Although all aforementioned methods perform reasonably well for a small number of dimensions, the higher dimensionality negatively affects their performance~\cite{tempczyk2022lidl,campadelli2015intrinsic,camastra2016intrinsic}. 

Very recently, a new approach emerged for the local methods, which is based on a two-step procedure. In the first step, the dataset is perturbed with some noise, and in the second step, its dimensionality is estimated using various techniques. These include generative models to analyze changes in density~\cite{tempczyk2022lidl,kamkari2024geometric} or in singular values of the Jacobian ~\cite{horvat2022intrinsic,horvat2024gauge} for different noise magnitudes, and analyzing rank of scores from diffusion model as recently presented by~\citet{stanczuk2024diffusion}.

While new algorithms provide state-of-the-art results for large, real-life datasets, it is vital for their performance that adding noise to the dataset should be as close as possible to an injective process.
This means that the resulting densities after adding noise should uniquely identify the original distribution of the dataset.
Otherwise, it is impossible to neither uniquely reverse the process nor analyze the original structure of the dataset.
One example in the theoretical model that does not admit such a problem is a flat manifold with a uniform distribution of data points.
In such cases these algorithms perform best, as shown in experimental results. 

In other cases, a hypothetical line of research would be to decrease the magnitude of noise added to the data set so that the manifold is approximately flat and has locally uniform data density in the scale of the noise magnitude.
This, however, is not possible in practice for several reasons.
For instance, data is often a set of discrete points (especially in audio and visual modality), and considering the noise of magnitude much smaller than the minimum distance between points does not lead to any meaningful process.
Moreover, neural net training is done with finite precision and the stochastic gradient descent procedure introduces noise to the density estimates, so very small changes in density cannot be observed in practice, which may lead to poor quality estimates of LID.

In this paper, we point out and exploit the fact that adding Gaussian noise of varying magnitudes can be seen as studying the evolution of the Wiener process describing the diffusion of particles (points of the dataset) in the ambient space.
This point of view enables us to employ Fick's Second Law of Diffusion to eliminate time derivatives from mathematical descriptions of state-of-the-art LID algorithms~\cite{tempczyk2022lidl,kamkari2024geometric}, and replace them with spatial derivatives.
Such considerations can be taken into account in the second step of the considered algorithms, leading to more accurate results.

\paragraph{Contribution.}
\begin{enumerate}
    \item We recognize and define new categories (isolated and holistic algorithms) for the Wiener process-based parametric LID estimation algorithm family and categorize the existing algorithms accordingly.
    \item We explore the first step of existing algorithms in the language of Wiener processes and calculate important cases of diffusion from lower-dimensional manifolds with non-uniform probability density into ambient space.
    \item We derive closed-form expressions for important parameters used in two state-of-the-art isolated LID estimation algorithms as a function of on-manifold density and manifold dimensionality, which can be viewed as closed-form expressions for deviation from a flat manifold with uniform distribution case.
\end{enumerate}

\section{Related work}
\label{sec:related}
The review of the non-parametric methods for local and global intrinsic dimension estimation can be found in the work of \citet{campadelli2015intrinsic}, or \citet{camastra2016intrinsic}.
\citet{tempczyk2022lidl} compared these methods on bigger datasets in terms of their dimensionality.

Although we do not analyze non-parametric methods in this paper, it is worth mentioning a recent work on non-parametric LID estimation, in which \citet{bailey2022local} explore the connection of LID to other well-known measures for complexity: entropy and statistical divergences, and develop new analytical expressions for these quantities in terms of LID.
Consequently, \citet{bailey2023relationships} establish the relationships for cumulative Shannon entropy, entropy power, Bregman formulation of cumulative Kullback-Leibler divergence, and generalized Tsallis entropy variants, and propose four new estimators of LID, based on nearest neighbor distances.

During the last few years many parametric methods for estimating LID emerged.  
\citet{zheng2022learning} prove that VAE are capable of recovering the correct manifold dimension and demonstrate methods to learn manifolds of varying dimensions across the data sample.
\citet{yeats2023adversarial} connect the adversarial vulnerability of score models with the geometry of the underlying manifold they capture.
They show that minimizing the Dirichlet energy of learned score maps boosts their robustness while revealing the LID of the underlying data manifold.

\paragraph{Wiener process-based algorithms.}
Regarding parametric methods, there is a group of algorithms, that have one thing in common: they simulate a Wiener process on a dataset and directly use some properties of time-evolving density to estimate LID.
We can divide those algorithms into the following three groups:
\begin{enumerate}
    \item \textbf{LIDL} \cite{tempczyk2022lidl} and its efficient and most accurate implementation using diffusion models called \textbf{FLIPD} \cite{kamkari2024geometric}.
    These algorithms use the rate of change of the probability density at point $x$ during the Wiener process to estimate LID at $x$.
    For small diffusion times $t$ the logarithm of a density is a linear function of a logarithm of $t$, and the proportionality coefficient is equal $d-D$ for small $t$, where $d$ is manifold density and $D$ is ambient space dimensionality.
    Our experiments with ID-NF (described below) and FLIPD show, that the latter is more scalable than ID-NF (and ID-DM) due to the high memory and computational complexity of SVD (which has to be calculated for each data point). Experiments from \cite{kamkari2024geometric} show that FLIPD is more accurate than NB (described below), which led us to the conclusion, that FLIPD is state-of-the-art in LID estimation among the most scalable algorithms.
    \item  \textbf{ID-NF} \cite{horvat2022intrinsic} (using normalizing flows), and the diffusion-using follow-up paper \textbf{ID-DM} \cite{horvat2024gauge} analyze how singular values of a Jacobian of a function transforming a standard normal distribution into a diffused dataset density at $x$ evolves during the Wiener process.
    \citeauthor{horvat2022intrinsic} observed that when transforming a $d$-dimensional manifold into a $D$-dimensional Gaussian we have to expand space more in the normal direction to manifold, especially for small diffusion times $t$.
    \item \textbf{NB} \cite{stanczuk2024diffusion}, which is an abbreviation from Normal Bundle (name used in \cite{kamkari2024geometric}).
    \citeauthor{stanczuk2024diffusion} observed that for small diffusion times $t$ the gradients of the logarithm of a diffused data density (score function) close to $x$ lies in the normal space to the manifold and use this fact to estimate LID at $x$.
    
\end{enumerate}

\section{Isolated and holistic algorithms}

In this work, we take a closer look at the Wiener process-based algorithms described in the last paragraph of Section~\ref{sec:related}.
Although all these algorithms apply a Wiener process to the dataset during their first phase, when looking at their second step, we can divide them into two groups: isolated (LIDL, FLIPD, NB) and holistic (ID-NF, ID-DM).
The intermediate results of the first group that are used for LID calculation use only the information about the local shape of the data probability density function (without normalization constant). We assume that the generative model approximates diffused data distribution $\rho_t$ perfectly.
Their estimates depend on the proximity of $\nabla\log\rho_t$ to $x$ (NB) or $d\log\rho_t/d\log t$ at $x$ (LIDL, FLIPD). 

This is a consequence of $\rho_t$ at $x$ being a function that -- in practice -- depends on the values of original data distribution $p_S$ in a ball of radius $r\approx4\sqrt{t}$ around $x$ in the data space.
The values of $p_S$ outside this ball do not matter in practice for isolated algorithms, because the diffused particles in the Wiener process can travel longer distances than $r$ with very low probability (less than $3.7\cdot10^{-5}$).
When we add some new data points to the dataset far away from $x$, it does not change the shape of $\rho_t$ close to $x$.
This operation only changes a normalization constant, which becomes 0 after taking a logarithm of $\rho_t$ and taking a derivative either w.r.t time or spatial variables.

The holistic algorithms work quite differently.
In the case of ID-NF and ID-DM, they calculate singular values of the Jacobian of the function $\zeta: \mathrm{R}^N \rightarrow \mathrm{R}^N$ transforming $\rho_t$ into a standard normal distribution.
This $\zeta$ function strongly depends on the entire shape of $\rho_t$.
As mentioned before, when we change $\rho_t$ far away from $x$ we do not change the shape of $\rho_t$ close to $x$.
The same is not true for $\zeta(x)$.
When we add many data points to the dataset far away from $x$ while we keep the latent distribution fixed, we have to change the way we compress and stretch the space to match the new distribution.
This property makes the analysis of the behavior of ID-NF and ID-DM much harder (maybe even impossible) if we want to take into account only the data density in the neighbourhood of $x$. 

To give an illustrative example of this behavior, one needs to imagine that our dataset is one-dimensional and consists of 10K points sampled from $\NN{0}{1}$.
Typically, we are transforming it into $\NN{0}{1}$, so $\zeta$ is just an identity function.
Now let's add to the dataset another 10K points sampled from $\NN{100}{1}$. 
As a consequence, we have to stretch our space in some areas to transform one density into another, whereas $\zeta$ changes along with its Jacobian.

\section{Wiener process perspective on LID estimation}

Wiener process is a stochastic process modeling particle diffusion. Its increments over disjoint time intervals are independent and normally distributed, with variance proportional to time increments. Since in the machine learning community the term \emph{diffusion} is already overloaded, we will stick to Wiener process when speaking of particle diffusion process.

In this section we present a new perspective on perturbing datasets, unifying the approaches seen in the algorithms presented by~\citet{tempczyk2022lidl,stanczuk2024diffusion,horvat2022intrinsic,horvat2024gauge,kamkari2024geometric}. As already mentioned, all these algorithms consist of two stages, the first of which amounts to perturbing the dataset with normally distributed random noise of fixed variance $t$. In the second stage, each of the algorithms utilizes the behavior of the perturbed density in the neighborhood of a fixed point under changes in the noise variance.

The first phase of each algorithm can be interpreted as applying the Wiener process to the points in the dataset. Afterward, the resulting set of points is used to train some type of generative model (or models) to estimate the distribution of the dataset undergoing the Wiener process at time $t$.
From the point of view of differential equations, the distribution density function of the diffused dataset is described by Fick's Second Law of Diffusion.

\begin{FSLD}  
Let \( \rho_t:\mathbb{R}^D \mapsto \mathbb{R} \) denote the probability density function modeling particles undergoing diffusion at time $t$. Then $\rho_t$ satisfies the differential equation
\begin{equation}
    \label{eq:heat-equation}
    \frac{d}{dt}\rho_t = C \Delta \rho_t,
\end{equation}
where \( C\in\RR \), and \( \Delta \) stands for the standard Laplacian in $\RR^D$.
\end{FSLD}

Now, given a dataset embedded in $\RR^D$, we assume that it has been drawn from some latent union of submanifolds $S$ endowed with a probability measure $p_S$ (which can be naturally treated as a probability measure on $\RR^D$).
The goal of Local Intrinsic Dimension estimation is to find out the dimension of $S$ at any point of the dataset.

To model the Wiener process with initial distribution $p_S$ (which is not a function on $\RR^D$), let us first define
\begin{equation}
    \phi_t^D(x)=(2\pi t)^{-D/2} e^{-\norm{x}^2/2t}.
\end{equation}
This is the density of normal distribution on $\RR^D$ with covariance matrix $tI$. 
It is the fundamental solution of the differential equation given by Fick's Second Law of Diffusion~\eqref{eq:heat-equation} with $C=1/2$. Here, this means that the convolution
\begin{equation}
    \label{eq:diffusion-solution-general}
    \rho_t = p_S * \phi_t^D
\end{equation}
is the solution of~\eqref{eq:heat-equation} for $t>0$ and hence it describes the Wiener process starting from the initial probability distribution $p_S$.

To limit the complexity introduced by curvature, from now on we will consider only flat manifolds. This means that, without loss of generality, we may assume that $S$ is the first factor in product decomposition $\RR^D = \RR^d\times\RR^{D-d}$. We will denote the coordinates of $\RR^d$ and $\RR^{D-d}$ by $x$ and $y$, respectively. We will moreover assume that $p_S$, now a probability distribution on $\RR^d$, has a density $\psi\colon\RR^d\to\RR$.

In Appendix~\ref{app:laplacian-diffused-density} we discuss the Laplacian of $\rho_t$ and derive the following result:

\begin{lemma}[Lemma~\ref{lem:app:laplacian-rho-t-everywhere}]
    \label{lem:laplacian-rho-t-everywhere}
    For $t>0$ and $(x,y) \in \RR^D$ we have
    \begin{equation}
        \label{eq:lem:laplacian-rho-t-off-manifold}
        \begin{split}
            \Delta \rho_t(x, y) & = \left(\frac{\norm{y}^2}{t^2} + \frac{d-D}{t}\right)\rho_t(x, y)\\
            & + \phi_t^{D-d}(y) \Delta_x(\psi * \phi_t^d)(x).            
        \end{split}
    \end{equation}
\end{lemma}

As a consequence, by putting $y=0$ and using 
\begin{equation}
    \phi_t^{D-d}(0) = (2\pi t)^{(d-D)/2}
\end{equation} 
we obtain the following. 

\begin{corollary}
    For $t>0$ and $x\in \RR^d$ we have
    \begin{equation}
        \Delta \rho_t(x, 0) = \frac{d-D}{t}\rho_t(x, 0) + (2\pi t)^{(d-D)/2} \Delta_x(\psi * \phi_t^d)(x). 
    \end{equation}
\end{corollary}

\section{From Wiener process to LID estimation}
The findings from the last section can be used to analyze how the first family of algorithms \citep{tempczyk2022lidl,kamkari2024geometric} behaves for some particular cases. 
The main contribution of \citet{kamkari2024geometric} is a substantial improvement on the side of density estimation.
Therefore, when dealing with perfect density estimators and very small noise differences, both algorithms estimate the same quantity and give the same results from the theoretical perspective.
Due to this fact from now on we will be analyzing LIDL, as we want to analyze the aspects of those implementations that do not depend on the problems with density estimation itself.

\subsection{Reformulating LIDL}
\label{sec:reformulating-lidl}

Given a point $x\in S$ and a set of times $t_1,\dots, t_n$, LIDL estimates the linear regression coefficient $\alpha$ of the set of points $(\log\delta_i, \log\rho_{t_i}(x))$, where $\delta_i = \sqrt{t_i}$.
\citet{tempczyk2022lidl} proved that 
\begin{equation}
    \label{eq:lidl}
    \log\rho_t(x) = (d-D)\log \sqrt{t} + O(1),
\end{equation}
and therefore $\alpha \approx d-D$. The authors show that if $t$ is small enough, this estimate is accurate.

This procedure can be seen as approximating the asymptotic slope of the parametric curve $(\log \sqrt{t}, \log\rho_t(x))$.
In other words, the graph of $s\mapsto \log\rho_{e^{2s}}(x)$ for $s\to-\infty$.
Another approach would consider the its derivative.
Let us define its reparameterized derivative (with $t=e^{2s})$
\begin{equation}
    \label{eq:def-beta-t}
    \beta_t(x) = \frac{2t}{\rho_{t}(x)} \frac{d}{dt}\rho_t(x) 
        = \frac{t \Delta\rho_t(x)}{\rho_{t}(x)},
\end{equation}
where the last equality comes from the diffusion equation~\eqref{eq:heat-equation} with $C=1/2$. Moreover, denote the asymptotic slope of the aforementioned curve by

\begin{equation}
    \label{eq:def-beta}
    \beta(x) = \lim_{s\to -\infty} \frac{d}{ds} \log\rho_{e^{2s}}(x) 
    =  \lim_{t\to 0^+} \beta_t(x).
\end{equation}

The results presented below are proved in Appendix~\ref{app:equivalent-formulations-of-lidl}. The next Proposition shows that the two approaches discussed above are equivalent. 
\begin{proposition}[Proposition~\ref{prop:app:equivalent-asymptotics}]
    \label{prop:equivalent-asymptotics}
    Given a strictly positive differentiable function $f\colon (0,a) \to (0,\infty)$ and a positive real number $\alpha > 0$, the following conditions are equivalent.
    \begin{enumerate}
        \item The function $f$ explodes at $0$ like $t^{-\alpha}$, i.e.\ for some positive constants $c,C > 0$ one has $c < t^{\alpha}f(t) < C$ for some $\epsilon>0$ and $t\in (0,\epsilon)$.
        \item $\log f(t) = -\alpha\log t + O(1)$.
        \item $\lim_{t\to 0^+} {\log f(t)}/{\log t} = -\alpha$.
        \item $\lim_{t\to 0^+} {tf'(t)}/{f(t)} = -\alpha$.
    \end{enumerate}
\end{proposition}

As a consequence, the estimation of Local Intrinsic Dimension using LIDL can be achieved by computing $\beta(x)$, yielding $d=D+\beta(x)$.
\begin{proposition}[Proposition~\ref{prop:app:equivalent-lidl}]
    For $t$ near $0$ the following estimate holds
    \begin{equation}
        \log\rho_t(x) = \beta(x)\log{\sqrt{t}} + O(1).
    \end{equation}  
\end{proposition}

The next proposition provides an elegant expression for $\beta_t(x)$, and consequently for $\beta(x)$, expressed in terms of the density $\psi$ on $\RR^d$.

\begin{proposition}[Proposition~\ref{prop:app:beta_final}]
\label{prop:beta_final}
    For $t>0$ and $x\in S = \RR^d \subseteq \RR^D$ we have
    \begin{equation}
        \label{eq:time-derivative-of-rho-t}
        \beta_t(x) = d-D +  \frac{\Delta_x(\psi * \phi_t^d)(x)}{\psi * \phi_t^d(x)} \cdot t.
    \end{equation}
\end{proposition}

\subsection{LIDL Examples}
\label{sec:examples}
From the theoretical considerations of~\citet{tempczyk2022lidl} it follows that $\beta(x)=d-D$ if $\psi$ is sufficiently regular and positive near $x$.
In other words,
\begin{equation}
    \lim_{t\to 0^+} \frac{\Delta_x(\psi *  \phi_t^d)(x)}{\psi * \phi_t^d(x)} \cdot t = 0.
\end{equation}
Now, we will try to obtain this conclusion directly and calculate bias of LIDL for $t>0$ in a few special cases by analyzing the behavior of $\beta_t(x)$.

\paragraph{The ``uniform distribution'' on Euclidean space.}
There is no such thing as the uniform distribution on $\RR^d$. However, from a purely theoretical viewpoint, in our differential equation approach we don't need the assumption of $\phi$ being a probability density; it could be any function. And since constant functions are usually the simplest examples, we will now investigate what happens if we put $\psi(x) \equiv 1$ on the whole $\RR^d$ space.

Using Proposition~\ref{prop:beta_final} and the fact that $\psi$ has bounded derivatives, this case leaves us with
\begin{equation}
    \beta_t(x) = d-D + \frac{\Delta_x\psi * \phi_t^d(x)}{\psi * \phi_t^d(x)}\cdot t = d-D,
\end{equation}
since $\Delta_x\psi \equiv 0$. This expression is constant in $t$, and in particular its limit at $0$ is $\beta(x)=d-D$. In this case, LIDL estimator is not biased for all $t>0$.

\paragraph{Normal distribution.}
Now consider the normal distribution on $\RR^d$ with covariance matrix $\Sigma = \mathop{\operatorname{diag}}(\sigma_1^2,\dots,\sigma_d^2)$, and denote its density function by $\psi$.
The convolution $\psi * \phi_t^d$ is the density of the normal distribution with covariance matrix $\Sigma + tI$. If we simplify notation be putting $\phi_i = \phi_{\sigma_i^2+t}^1$, we get
\begin{equation}
    \psi*\phi_t^d(x) = \prod_{i=1}^d \phi_i(x_i).
\end{equation}

To compute the Laplacian of this convolution, note that 
\begin{equation}
    \psi*\phi_t^d(x) = \frac{\psi*\phi_t^d(x)}{\phi_i(x_i)} \cdot \phi_i(x_i),
\end{equation}
where the first factor does not depend of $x_i$, and therefore
\begin{equation}
    \frac{\partial^2 (\psi*\phi_t^d)}{\partial x_i^2} (x) = \psi(x)*\phi_t^d(x) \cdot \frac{1}{\phi_i(x_i)}\frac{\partial^2\phi_i}{\partial x_i^2}(x_i),
\end{equation}
leading to 
\begin{equation}
    \begin{split}
        \beta_t(x) &= d - D + t\frac{\Delta (\psi*\phi_t^d)(x)}{\psi*\phi_t^d(x)} 
        = \\
        & = d - D + t\sum_{i=1}^d \frac{1}{\phi_i(x_i)}\frac{\partial^2\phi_i}{\partial x_i^2}(x_i) \\
        & = d-D + t\sum_{i=1}^d\frac{x^2_i - (\sigma_i^2 + t)}{(\sigma_i^2 + t)^2}.
    \end{split}
\end{equation}
It is easy to see that the second derivatives of $\phi_i$ are continuous in $t> -\sigma_i^2$, so the sum in the above expression has finite limit for $t\to 0$, and therefore $\beta(x)=d-D$.

In the special case where $\Sigma = \sigma^2 I$, these calculations simplify further, as $\psi * \phi_t^d = \phi_{\sigma^2+t}^d$, and since
\begin{equation}
    \Delta_x\phi_{\sigma^2+t}^d(x) = \left( \frac{\norm{x}^2}{(\sigma^2+t)^2} - \frac{d}{\sigma^2+t}\right) \phi_{\sigma^2+t}^d(x),
\end{equation}
we have
\begin{equation}
    \label{eq:beta-t-normal-dist}
    \beta_t(x) = d-D + \left( \frac{\norm{x}^2}{(\sigma^2+t)^2} - \frac{d}{\sigma^2+t}\right)t.
\end{equation}

\begin{figure}[t]
    \centering
    \begin{subfigure}[b]{0.47\textwidth}
        \centering
        \includegraphics[width=\textwidth]{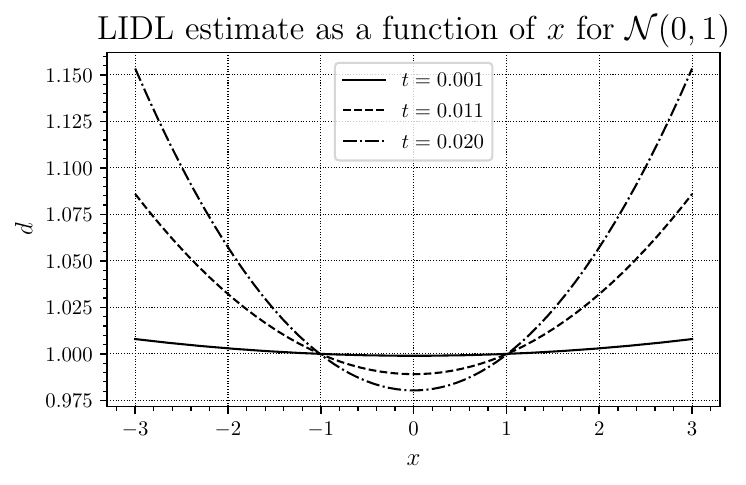}
        \caption{Example of the bias of a LIDL estimate for different points from $\mathcal{N}(0,1)$ and for different values of $t$. This plot recreates a numerical calculations presented in Figure~4 from \cite{tempczyk2022lidl}, with two minor differences described in Sec.~\ref{sec:examples}}
        \label{fig:plot_parabola}
    \end{subfigure}
    \hfill
    \begin{subfigure}[b]{0.47\textwidth}
        \centering
        \includegraphics[width=\textwidth]{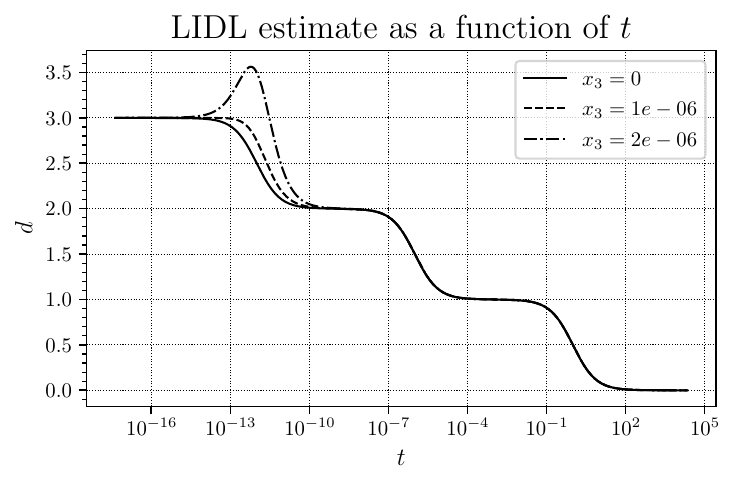}
        \caption{Plot of a LIDL estimate as a function of $t$ for the distribution $\mathcal{N}(\mathbf{0}, \mathop{\operatorname{diag}}(1,10^{-6},10^{-12}))$ and for three different points $\mathbf{x} = (0,0,x_3)$, which represents a distance of 0, 1 and 2 $\sigma_3$ from~0 on 3rd dimension (compare with Fig.~8 from FLIPD and Fig.~2 from LIDL).}
        \label{fig:plot_stairs}
    \end{subfigure}
    \caption{LIDL estimates for Gaussian distributions.}
    \label{fig:combined_plots}
\end{figure}

These results express analytically the experimental observations from LIDL and FLIPD papers, as can be verified by looking at Fig.~\ref{fig:plot_parabola}.
We can observe, that if we move to the regions of very low probability for a Gaussian, it generates very high positive bias, which may highly overestimate the true LID (also observed as a \emph{bump} at $t=10^{-12}$ in Fig.~\ref{fig:plot_stairs}).
Luckily, most of the points in our dataset come from the region of high probability, but we should be less certain of the estimates for points from low probability regions.

It is worth noting, that the values of $t$ used to generate Fig.~\ref{fig:plot_parabola} are the same as values of $\delta$, which is equal to $\sqrt{t}$ in our convention.
After double-checking our results we argue that the most probable cause of this is that the authors of LIDL used squared values of $\delta$ by mistake.
Additionally, one can observe that curves plotted by \citet{tempczyk2022lidl} are somewhat flatter than one in this study.
The fact that in their paper the derivative was approximated by linear regression on numerically calculated densities -- which may lead to slightly different results -- might be a possible reason.

\paragraph{Arbitrary distribution with sufficiently \emph{nice} density.}
By this point, the notion of \emph{nice} density is shall be more clear.
We want to be able to use the equality 
\begin{equation}
\Delta_x(\psi * \phi_t^d) = \Delta_x\psi * \phi_t^d.
\end{equation}
To do so, we need $\psi$ to be bounded, twice differentiable, and have bounded first and second-order partial derivatives.
We will also require $\psi$ to have \emph{continuous} second-order partial derivatives.
This is not a severe restriction, as numerous distributions satisfy these properties -- including the normal distribution or more generally, mixtures of Gaussians.

In this case, we have
\begin{equation}
    \beta_t(x) = d-D + \frac{\Delta_x\psi * \phi_t^d(x)}{\psi * \phi_t^d(x)} \cdot t,
\end{equation}
however this time $\Delta_x\psi$ is some arbitrary continuous function. Being differentiable, $\psi$ is also continuous, and we can use the general fact that for a bounded continuous function, $f$ on $\RR^d$ one has
\begin{equation}
    \lim_{t\to 0^+} f*\phi_t^d(x) = f(x).
\end{equation}
This gives us, for $x$ such that $\psi(x)>0$, 
\begin{equation}
    \lim_{t\to 0^+}  \frac{\Delta_x\psi * \phi_t^d(x)}{\psi * \phi_t^d(x)} \cdot t 
    = \frac{\Delta_x\psi(x)}{\psi(x)} \lim_{t\to 0^+} t
    = 0,
\end{equation}
and again $\beta(x) = d-D$.
It has been already proven that in this case $\beta$ yields a correct estimate of dimension, circumventing complexities of \citet{tempczyk2022lidl} proofs.

It is worth noting, that when $\Delta_x\psi = 0$, the estimate is accurate.
It is the case for the aforementioned ``uniform distribution'' on $\RR^d$, but it is also true if locally the density is a linear function of $x$.
In Fig.~\ref{fig:plot_parabola} we can observe that for $x\approx\pm1$ (Laplacian of a Gaussian density equals 0 at these points) and small values of $t$, the estimate is accurate.

\paragraph{Uniform distribution supported on an interval.}
Now consider an example where the density is not differentiable -- the uniform distribution on an interval $[a, b]\subset\RR$, i.e.
\begin{equation}
    \psi(x) = \frac{1}{b-a} \chi_{[a, b]}(x),
\end{equation}
where $\chi_A(s)$ is the indicator function of the set $A$, equal to $1$ on $A$ and $0$ outside $A$.
In the next example, we will generalize this to a hypercube, but the core observations can be made in this simpler 1-dimensional case.

The difficulty introduced by the non-differentiability of $\psi$ is we are no longer allowed to move the Laplacian inside the convolution to get 
\begin{equation}
    \Delta_x(\psi * \phi_t) = \Delta_x\psi * \phi_t
\end{equation} (we omit the superscript $d=1$ from $\phi_t$) -- as tempting as it might be.
Therefore, a different manner of proceeding is needed.
We may still move the Laplacian to $\phi_t$.
In the 1-dimensional case, $\Delta_x$ is simply the second derivative, and since
\begin{equation}
    \phi_t'(u)=-u\phi_t(u)/t
\end{equation}
we have
\begin{equation}
    \begin{split}
        \Delta_x(\psi * \phi_t)(x) 
        & = \frac{1}{b-a}\int_{x-b}^{x-a} \phi_t''(u)\,du \\
        & = \frac{\phi_t'(x-a) - \phi_t'(x-b)}{b-a} = \\
        & = \frac{(x-b)\phi_t(x-b)-(x-a)\phi_t(x-a)}{t(b-a)},
    \end{split}
\end{equation}
Expanding the denominator in a similar fashion yields
\begin{equation}
    \beta_t(x) = d - D + \frac{(x-b)\phi_t(x-b)-(x-a)\phi_t(x-a)}{\Phi_t(x-a) - \Phi_t(x-b)},
\end{equation}
where $\Phi_t$ is the cumulative distribution function corresponding to the density $\phi_t$.
In particular for $x\in(a,b)$ we see that since $x-b < 0 < x-a$, when $t\to 0^+$, the denominator tends to $1$, while both terms of the numerator tend to $0$, leaving us with $d-D$.
LIDL estimate curves for this case for different values of $t$ are plotted in Fig.~\ref{fig:plot_uniform}.

\begin{figure}[t]
    \centering
    \begin{subfigure}[b]{0.47\textwidth}
        \centering
        \includegraphics[width=\textwidth]{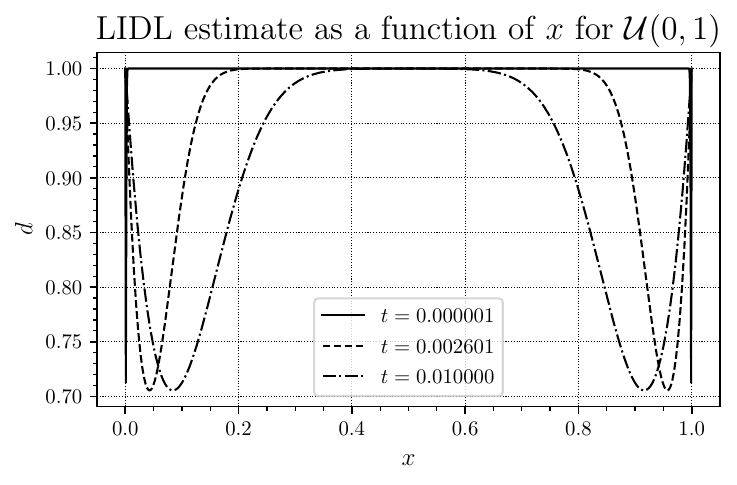}
        \caption{Example of the bias of a LIDL estimate for different points from $\mathcal{U}(0,1)$ and values of $t$. This plot recreates a numerical calculations presented in Figure~3 from \cite{tempczyk2022lidl}}
        \label{fig:plot_uniform}
    \end{subfigure}
    \hfill
    \begin{subfigure}[b]{0.47\textwidth}
        \centering
        \includegraphics[width=\textwidth]{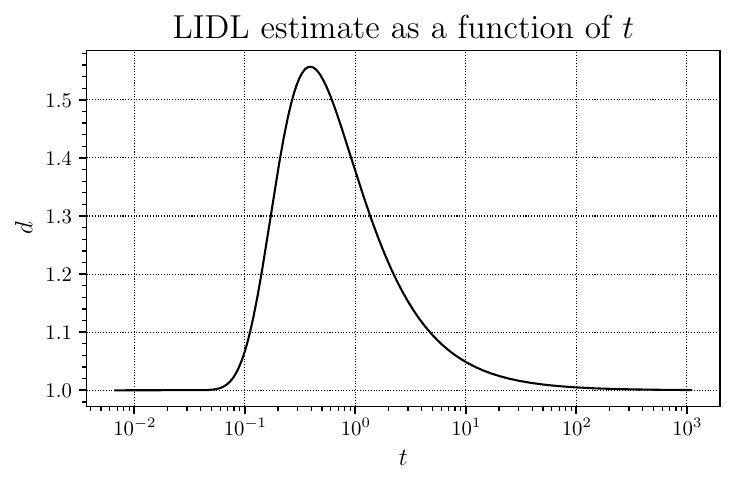}
        \caption{LIDL estimate as a function of t for a point from parallel 1D manifolds separated by a distance of 1 with uniform distribution on them. Similar to result from Fig.~6 in \cite{tempczyk2022lidl}.}
        \label{fig:plot_parallel}
    \end{subfigure}
    \caption{LIDL estimates.}
    \label{fig:combined_plots_uniform_parallel}
\end{figure}

\paragraph{Uniform distribution supported on a hypercube.}
Let us now consider a more general case -- the uniform distribution on a hypercube $[a_1, b_1] \times \dots\times [a_d, b_d] \subset \RR^d$.
We have 
\begin{equation}
    \psi(x) = \prod_{i=1}^d \frac{1}{b_i-a_i} \chi_{[a_i, b_i]}(x_i),
\end{equation}
Denote
\begin{equation}
    \psi_i(s) = \frac{1}{b_i-a_i} \chi_{[a_i, b_i]}(s),
\end{equation}
and observe that since $\phi_t^d(x)$ is the product of $\phi_t(x_i)$, we have
\begin{equation}
    \psi * \phi_t^d(x) = \prod_{i=1}^d \psi_i * \phi_t(x_i).
\end{equation}
By directly computing the derivatives, we obtain
\begin{equation}
    \frac{\Delta_x\psi * \phi_t^d(x)}{\psi * \phi_t^d(x)} = \sum_{i=1}^d \frac{(\psi_i * \phi_t)''(x_i)}{\psi_i * \phi_t(x_i)} = \sum_{i=1}^d \frac{\psi_i * \phi_t''(x_i)}{\psi_i * \phi_t(x_i)},
\end{equation}
reducing our problem to the $1$-dimensional variant we have dealt with in the preceding example. Summing up, we have
\begin{equation}
    \begin{split}
        \beta_t(x) & = d - D\\
        & + \sum_{i=1}^d \frac{(x_i-b_i)\phi_t(x_i-b_i)-(x_i-a_i)\phi_t(x_i-a_i)}{\Phi_t(x_i-a_i) - \Phi_t(x_i-b_i)},    
    \end{split}
\end{equation}
and the asymptotic behavior with $t\to 0^+$ follows the $1$-dimensional case.

\paragraph{Union of two parallel hyperplanes.}
Suppose that $S$ is a union of two parallel hyperplanes, $S_1 = \RR^d$ and $S_2 = v + \RR^d$, where $v \perp \RR^d$.
Moreover, assume that $p_S = (1-\lambda) p_1 + \lambda p_2$ is a convex combination of probability measures $p_i$ supported on $S_i$, with densities $\psi_i\colon\RR^d\to\RR$ (we identify $S_2$ with $\RR^d$ through the map $x\mapsto x+v$).
In this case, for $x\in S_1$ we have
\begin{equation}
    \beta_t^1(x) = d-D + \frac{\Delta_x(\psi_1 * \phi_t^d)(x)}{\psi_1 * \phi_t^d(x)} \cdot t,
\end{equation}
and by Lemma~\ref{lem:laplacian-rho-t-everywhere}, and the observation that 
\begin{equation}
\rho_t^2(x) = \psi_2*\phi_t^d(x)\phi_t^{D-d}(v),    
\end{equation}
we get
\begin{equation}
    \beta_t^2(x) = d-D + \frac{\norm{v}^2}{t} + \frac{\Delta_x(\psi_2*\phi_t^d)(x)}{\psi_2*\phi_t^d(x)}\cdot t.
\end{equation}

Here, we see that for an off-manifold point $x\not\in S_2$, the expression for $\beta_t^2(x)$ contains a summand $\norm{v}^2/t$ that explodes at $0$, and if the last term is under control, $\beta^2(x)$ is infinite.
However, by Lemma~\ref{lem:convex-combination-coeff-estimate}, the coefficient of $\beta_t^2(x)$ in the expansion of $\beta_t(x)$ from Lemma~\ref{lem:convex-combination} decreases exponentially in $1/t$, neutralizing this divergence.

For the remainder of this example, let us assume $\psi_1 = \psi_2 = \psi$. In this case, we may apply multiple simplifications, in particular
\begin{equation}
    \beta_t^2(x) = \beta_t^1(x) + \frac{\norm{v}^2}{t},
\end{equation}
and moreover
\begin{equation}
    \begin{split}
        \frac{\rho_t^2(x)}{\rho_t(x)} &= \\
        =& \frac{\psi * \phi_t^d(x)\phi_t^{D-d}(v)}{(1-\lambda)\psi * \phi_t^d(x)\phi_t^{D-d}(0) + \lambda\psi * \phi_t^d(x)\phi_t^{D-d}(v)} \\
        =& \frac{\phi_t^{D-d}(v)}{(1-\lambda)\phi_t^{D-d}(0) + \lambda\phi_t^{D-d}(v)} \\
        =& \frac{1}{(1-\lambda)e^{\norm{v}^2/2t} + \lambda}.
    \end{split}
\end{equation}
Since $\frac{\lambda \rho_t^1(x)}{\rho_t(x)} =1-\frac{\lambda \rho_t^2(x)}{\rho_t(x)} $ we can simplify the expression for $\beta_t(x)$ from Lemma~\ref{lem:convex-combination}, yielding
\begin{equation}
    \begin{split}
    \beta_t(x) &= \beta_t^1(x) + \frac{\norm{v}^2}{t} \cdot \frac{\lambda \rho^2_t(x)}{\rho_t(x)} \\
    & = \beta_t^1(x) + \frac{\lambda\norm{v}^2}{t\left( (1-\lambda)e^{\norm{v}^2/2t} + \lambda \right)}.
    \end{split}
\end{equation}
To give a concrete example, if $\psi = \phi_{\sigma^2}^d$, then 
\begin{equation}
    \begin{split}
        \beta_t(x) &= d-D + \left( \frac{\norm{x}^2}{(\sigma^2+t)^2} - \frac{d}{\sigma^2+t}\right)t \\
        &+ \frac{\lambda\norm{v}^2}{t\left( (1-\lambda)e^{\norm{v}^2/2t} + \lambda \right)}.    
    \end{split}
\end{equation}
Another interesting example occurs when the data on both parallel manifolds follow uniform density functions. 
Although the derivation in Eq.~\eqref{eq:convex_comb_of_beta} requires the density functions to be probability distributions, this scenario can be simulated by considering two Gaussian distributions with relatively large standard deviations.
This yields the following formula for $\beta_t(x)$, presented in Fig.~\ref{fig:plot_parallel} for $v=1$, $\lambda=\frac{1}{2}$:
\begin{equation}
    \beta_t(x) = d-D + \frac{\lambda\norm{v}^2}{t\left( (1-\lambda)e^{\norm{v}^2/2t} + \lambda \right)}.
\end{equation}

\paragraph{Union of two intersecting manifolds.}
In this example we will consider a manifold $S$ decomposing into a union of two components $S_1$ and $S_2$, intersecting at a point $x\in S_1 \cap S_2$. Denote by $d_i$ the dimension of $S_i$. As before, let $p_S = \lambda p_1 + (1-\lambda)p_2$. Moreover, suppose that
\begin{equation}
    \beta_t^i(x) = d_i - D + E_i(t),
\end{equation}
where $E_i(t)$ expresses the error of $\beta_t^i$ in estimating the dimension of $S_i$, and $\lim_{t\to 0^+}E_i(t) = 0$. By Lemma~\ref{lem:convex-combination} we have
\begin{equation}
    \beta_t(x) 
    = \left(\frac{\lambda \rho_t^1(x)}{\rho_t(x)}d_1 + \frac{(1-\lambda) \rho_t^2(x)}{\rho_t(x)}d_2 \right) 
    - D + E(t), 
\end{equation}
where the error term is a convex combination of $E_1$, and $E_2$, and thus is bounded by their maximum,
\begin{equation}
    \begin{split}
    E(t) &= \frac{\lambda \rho_t^1(x)}{\rho_t(x)}E_1(t) + \frac{(1-\lambda) \rho_t^2(x)}{\rho_t(x)}E_2(t)\\ &\leq \max\{ E_1(t), E_2(t)\}.
    \end{split}
\end{equation}
In particular, it also vanishes as $t\to 0^+$. 

The value of LID at $x$ estimated by $\beta_t(x)$ lies between $d_1$ and $d_2$, and is controlled by the asymptotic of $\lambda\rho_t^1(x)/\rho_t(x)$. If $d_1=d_2=d$, then it is also equal to $d$.

\section{Conclusions}

In this work, for the first time, we have outlined a new perspective for Wiener process-based algorithms for LID estimation and shown some results for LIDL and FLIPD, in which we found an analytical description for the phenomena observed in experiments. 

The presented results open up several promising new research directions. A natural extension of this study would involve accounting for the curvature of manifolds, addressing the current assumption of manifold flatness. 
Moreover, an interesting direction for future research could be an extended analysis of how nearby manifolds can affect LID estimates as briefly shown in the example of parallel manifolds in Sec.~\ref{sec:examples}.
Another potential research avenue is to apply the proposed approach to analyze other algorithms, such as the NB algorithm, in a manner similar to LIDL.

Finally, our brief experiments show that using density models to estimate Laplacian can be beneficial in the process of improving LIDL estimate for higher values of $t$ and non-uniform densities, leading to a promising future direction of research. Using Wiener process perspective for answering the question how dataset quantization can affect the estimate is another interesting question to answer in the area of LID estimation.

\section*{Acknowledgments}
We want to thank Marek Cygan, Micha\l{} Karpowicz and Adam Goli\'nski for their overall support during this project.

\paragraph{CRediT Author Statement.}
\textbf{Piotr Tempczyk} (32\% of the work): Conceptualization, Methodology, Software, Validation, Formal analysis, Investigation, Data Curation, Writing -- Original Draft, Writing -- Review \& Editing, Visualization, Supervision, Project administration.
\textbf{Adam Kurpisz} (32\% of the work): Conceptualization, Methodology, Validation, Formal analysis, Investigation, Writing -- Original Draft, Writing -- Review \& Editing, Visualization, Supervision.
\textbf{Łukasz Garncarek} (26\% of the work): Conceptualization, Methodology, Formal analysis, Investigation, Writing -- Original Draft, Writing -- Review \& Editing.
\textbf{Dominik Filipiak} (10\% of the work): Software, Data Curation, Writing -- Review \& Editing.

\bibliography{aaai25}

\begin{thebibliography}{25}
\providecommand{\natexlab}[1]{#1}

\bibitem[{Ansuini et~al.(2019)Ansuini, Laio, Macke, and Zoccolan}]{ansuini2019intrinsic}
Ansuini, A.; Laio, A.; Macke, J.~H.; and Zoccolan, D. 2019.
\newblock Intrinsic dimension of data representations in deep neural networks.
\newblock In \emph{Advances in Neural Information Processing Systems}, 6111--6122.

\bibitem[{Bailey, Houle, and Ma(2022)}]{bailey2022local}
Bailey, J.; Houle, M.~E.; and Ma, X. 2022.
\newblock Local intrinsic dimensionality, entropy and statistical divergences.
\newblock \emph{Entropy}, 24(9): 1220.

\bibitem[{Bailey, Houle, and Ma(2023)}]{bailey2023relationships}
Bailey, J.; Houle, M.~E.; and Ma, X. 2023.
\newblock Relationships between tail entropies and local intrinsic dimensionality and their use for estimation and feature representation.
\newblock \emph{Information Systems}, 118: 102245.

\bibitem[{Brown et~al.(2022)Brown, Caterini, Ross, Cresswell, and Loaiza-Ganem}]{brown2022verifying}
Brown, B.~C.; Caterini, A.~L.; Ross, B.~L.; Cresswell, J.~C.; and Loaiza-Ganem, G. 2022.
\newblock Verifying the Union of Manifolds Hypothesis for Image Data.
\newblock In \emph{The Eleventh International Conference on Learning Representations}.

\bibitem[{Camastra and Staiano(2016)}]{camastra2016intrinsic}
Camastra, F.; and Staiano, A. 2016.
\newblock Intrinsic dimension estimation: Advances and open problems.
\newblock \emph{Information Sciences}, 328: 26--41.

\bibitem[{Campadelli et~al.(2015)Campadelli, Casiraghi, Ceruti, and Rozza}]{campadelli2015intrinsic}
Campadelli, P.; Casiraghi, E.; Ceruti, C.; and Rozza, A. 2015.
\newblock Intrinsic dimension estimation: Relevant techniques and a benchmark framework.
\newblock \emph{Mathematical Problems in Engineering}, 2015: 1--21.

\bibitem[{Fan et~al.(2010)Fan, Gu, Qiao, and Zhang}]{fan2010}
Fan, M.; Gu, N.; Qiao, H.; and Zhang, B. 2010.
\newblock Intrinsic dimension estimation of data by principal component analysis.
\newblock \emph{arXiv preprint arXiv:1002.2050}.

\bibitem[{Fukunaga and Olsen(1971)}]{Fukunaga1971}
Fukunaga, K.; and Olsen, D.~R. 1971.
\newblock An Algorithm for Finding Intrinsic Dimensionality of Data.
\newblock \emph{IEEE Transactions on Computers}, C-20: 176--183.

\bibitem[{Horvat and Pfister(2022)}]{horvat2022intrinsic}
Horvat, C.; and Pfister, J.-P. 2022.
\newblock Intrinsic dimensionality estimation using normalizing flows.
\newblock \emph{Advances in Neural Information Processing Systems}, 35: 12225--12236.

\bibitem[{Horvat and Pfister(2024)}]{horvat2024gauge}
Horvat, C.; and Pfister, J.-P. 2024.
\newblock On gauge freedom, conservativity and intrinsic dimensionality estimation in diffusion models.
\newblock \emph{arXiv preprint arXiv:2402.03845}.

\bibitem[{Johnsson, Soneson, and Fontes(2014)}]{johnsson2014low}
Johnsson, K.; Soneson, C.; and Fontes, M. 2014.
\newblock Low bias local intrinsic dimension estimation from expected simplex skewness.
\newblock \emph{IEEE transactions on pattern analysis and machine intelligence}, 37(1): 196--202.

\bibitem[{Kamkari et~al.(2024{\natexlab{a}})Kamkari, Ross, Cresswell, Caterini, Krishnan, and Loaiza-Ganem}]{pmlr-v235-kamkari24a}
Kamkari, H.; Ross, B.~L.; Cresswell, J.~C.; Caterini, A.~L.; Krishnan, R.; and Loaiza-Ganem, G. 2024{\natexlab{a}}.
\newblock A Geometric Explanation of the Likelihood {OOD} Detection Paradox.
\newblock In Salakhutdinov, R.; Kolter, Z.; Heller, K.; Weller, A.; Oliver, N.; Scarlett, J.; and Berkenkamp, F., eds., \emph{Proceedings of the 41st International Conference on Machine Learning}, volume 235 of \emph{Proceedings of Machine Learning Research}, 22908--22935. PMLR.

\bibitem[{Kamkari et~al.(2024{\natexlab{b}})Kamkari, Ross, Hosseinzadeh, Cresswell, and Loaiza-Ganem}]{kamkari2024geometric}
Kamkari, H.; Ross, B.~L.; Hosseinzadeh, R.; Cresswell, J.~C.; and Loaiza-Ganem, G. 2024{\natexlab{b}}.
\newblock A Geometric View of Data Complexity: Efficient Local Intrinsic Dimension Estimation with Diffusion Models.
\newblock In \emph{ICML 2024 Workshop on Structured Probabilistic Inference {\&} Generative Modeling}.

\bibitem[{Kleindessner and Luxburg(2015)}]{kleindessner2015dimensionality}
Kleindessner, M.; and Luxburg, U. 2015.
\newblock Dimensionality estimation without distances.
\newblock In \emph{Artificial Intelligence and Statistics}, 471--479.

\bibitem[{Levina and Bickel(2004)}]{levina2004maximum}
Levina, E.; and Bickel, P. 2004.
\newblock Maximum likelihood estimation of intrinsic dimension.
\newblock \emph{Advances in neural information processing systems}, 17.

\bibitem[{Li et~al.(2018)Li, Farkhoor, Liu, and Yosinski}]{li2018measuring}
Li, C.; Farkhoor, H.; Liu, R.; and Yosinski, J. 2018.
\newblock Measuring the Intrinsic Dimension of Objective Landscapes.
\newblock In \emph{International Conference on Learning Representations}.

\bibitem[{Loaiza-Ganem et~al.(2024)Loaiza-Ganem, Ross, Hosseinzadeh, Caterini, and Cresswell}]{loaiza2024deep}
Loaiza-Ganem, G.; Ross, B.~L.; Hosseinzadeh, R.; Caterini, A.~L.; and Cresswell, J.~C. 2024.
\newblock Deep Generative Models through the Lens of the Manifold Hypothesis: A Survey and New Connections.
\newblock \emph{arXiv preprint arXiv:2404.02954}.

\bibitem[{Minka(2000)}]{minka2000}
Minka, T. 2000.
\newblock Automatic choice of dimensionality for PCA.
\newblock \emph{Advances in neural information processing systems}, 13.

\bibitem[{Pope et~al.(2020)Pope, Zhu, Abdelkader, Goldblum, and Goldstein}]{pope2020intrinsic}
Pope, P.; Zhu, C.; Abdelkader, A.; Goldblum, M.; and Goldstein, T. 2020.
\newblock The Intrinsic Dimension of Images and Its Impact on Learning.
\newblock In \emph{International Conference on Learning Representations}.

\bibitem[{Rubenstein, Schoelkopf, and Tolstikhin(2018)}]{rubenstein2018latent}
Rubenstein, P.~K.; Schoelkopf, B.; and Tolstikhin, I. 2018.
\newblock On the latent space of wasserstein auto-encoders.
\newblock \emph{arXiv preprint arXiv:1802.03761}.

\bibitem[{Stanczuk et~al.(2024)Stanczuk, Batzolis, Deveney, and Sch{\"o}nlieb}]{stanczuk2024diffusion}
Stanczuk, J.~P.; Batzolis, G.; Deveney, T.; and Sch{\"o}nlieb, C.-B. 2024.
\newblock Diffusion Models Encode the Intrinsic Dimension of Data Manifolds.
\newblock In \emph{Forty-first International Conference on Machine Learning}.

\bibitem[{Tempczyk et~al.(2022)Tempczyk, Michaluk, Garncarek, Spurek, Tabor, and Golinski}]{tempczyk2022lidl}
Tempczyk, P.; Michaluk, R.; Garncarek, L.; Spurek, P.; Tabor, J.; and Golinski, A. 2022.
\newblock Lidl: Local intrinsic dimension estimation using approximate likelihood.
\newblock In \emph{International Conference on Machine Learning}, 21205--21231. PMLR.

\bibitem[{Vapnik(2013)}]{vapnik2013nature}
Vapnik, V. 2013.
\newblock \emph{The nature of statistical learning theory}.
\newblock Springer science \& business media.

\bibitem[{Yeats et~al.(2023)Yeats, Darwin, Liu, and Li}]{yeats2023adversarial}
Yeats, E.; Darwin, C.; Liu, F.; and Li, H. 2023.
\newblock Adversarial Estimation of Topological Dimension with Harmonic Score Maps.
\newblock \emph{arXiv preprint arXiv:2312.06869}.

\bibitem[{Zheng et~al.(2022)Zheng, He, Qiu, and Wipf}]{zheng2022learning}
Zheng, Y.; He, T.; Qiu, Y.; and Wipf, D.~P. 2022.
\newblock Learning manifold dimensions with conditional variational autoencoders.
\newblock \emph{Advances in Neural Information Processing Systems}, 35: 34709--34721.

\end{thebibliography}

\appendix
\section{Additional clarifications}

In this section, we present some more detailed clarifications that could not be put into the main body of the paper due to the page limitations and were raised by the reviewers during the review process.

\paragraph{Which LID estimation algorithms assumptions are not met?}
Any of the existing algorithms for LID estimation assume that they give perfect estimates when diffusion times $t \to 0$.
But in most real-world problems, this assumption cannot be fulfilled exactly.
For example, when working with Normalizing Flows we cannot proceed with noise to 0 ($t \to 0$), because the training explodes for small values of $t$.
The same goes for working with quantized datasets, e.g. images, where there is no point in going with $\sqrt{t}$ below the quantization step.

All the algorithms assume that for a sufficiently small $t$ the local data density and manifold have some properties, e.g. the manifold is flat.
For different algorithms and data distributions, these \emph{small enough} values are, in general, not the same, but for any algorithm and any value of $t$ we can find a data manifold and density for which this value is too big and introduces an error of some kind.

\paragraph{Lack of experimental results.}
Reader may wonder why there are no experiments in this paper. The last term of Eq.~{\eqref{eq:time-derivative-of-rho-t} is currently not known to be computable for arbitrary datasets.
For this reason, further experiments with real datasets are not possible at this moment.
However, addressing the question of the computability of this ratio would immediately lead to an improved algorithm over state-of-the-art methods such as LIDL and FLIPD.

\subsection{Definitions}

Here we provide extended definitions of terms used in this work, which may be useful while reading the paper.

\paragraph{Manifold.}
A manifold is a mathematical space that locally resembles Euclidean space but may exhibit a more complex global structure.
In the context of LID estimation, a manifold represents the underlying low-dimensional structure within the ambient space where data points reside.
For instance, in generative modeling and representation learning, understanding the manifold helps in uncovering the true degrees of freedom that characterize a dataset, such as variations in facial expressions or object orientations.
For example, a dataset of faces may vary along dimensions such as age, expression, or pose, forming a manifold within the ambient space of all possible images.

\paragraph{Ambient Space.}
Ambient space is the high-dimensional space in which a manifold is embedded, encompassing all possible configurations, whether realized in the given data or not.
In the context of LID estimation, the ambient space is integral to analyzing how noise addition and diffusion processes perturb the data, as these processes operate within this space.
For instance, in image data, the ambient space might represent the set of all possible pixel configurations, while the data itself resides on a low-dimensional manifold embedded within this space.

\paragraph{Intrinsic Dimension.}
Intrinsic dimension refers to the number of independent variables or degrees of freedom required to describe the data.
Unlike the dimensionality of the ambient space, the intrinsic dimension captures the actual complexity of the data, reflecting the true number of parameters needed to represent its structure.

\paragraph{Wiener Process.}
A Wiener process, also known as Brownian motion, is a stochastic process characterized by continuous paths and Gaussian increments that are independent over time.
Within the context of LID estimation and deep learning, the Wiener process provides a mathematical framework for analyzing the impact of noise addition on data and the evolution of diffusion processes in the ambient space.
It allows modeling how perturbations influence data density and structure, thereby enabling more precise estimation of intrinsic dimensions in high-dimensional datasets.

\paragraph{Normalizing Flow.}
Normalizing flow is a class of generative modeling techniques that transforms a simple probability distribution, such as a Gaussian, into a complex distribution through a series of invertible mappings.
In the context of analyzing high-dimensional data and intrinsic dimensions, normalizing flows interact with manifold assumptions by enabling precise density estimation and characterizing data perturbations.
These transformations allow efficient computation of both the data likelihood and inverse mapping, making them valuable for exploring density changes and understanding the underlying data structure in generative models and deep learning applications.

\paragraph{Denoising Diffusion Probabilistic Models.} DDPMs represtent a type of generative model which learns to reverse a diffusion process.
Starting from Gaussian noise, such a model iteratively denoises the data through a series of steps guided by a learned probability distribution.
These models have demonstrated state-of-the-art performance in generating high-quality data, such as images, and are grounded in principles of stochastic processes and energy-based modeling.

\section{Laplacian of the diffused density}
\label{app:laplacian-diffused-density}

Assuming that $p_S$, considered as a probability distribution on $\RR^d$, has a density $\psi\colon\RR^d\to\RR$, we can separate variables in~\eqref{eq:diffusion-solution-general}, obtaining
\begin{equation}
\label{eq:rho_t}
    \rho_t(x, y) = \psi * \phi_t^d(x) \phi_t^{D-d}(y)
\end{equation}
for $(x,y)\in\RR^d\times\RR^{D-d}$. Moreover, the decomposition $\RR^D = \RR^d\times\RR^{D-d}$ gives rise to the decomposition of the Laplacian on $\RR^D$ into Laplacians on the factors, namely $\Delta = \Delta_x + \Delta_y$, by simply grouping the derivatives with respect to $x_i$ and $y_i$ coordinates.

\begin{lemma}
    \label{lem:laplacian-rho-t-everywhere}
    For $t>0$ and $(x,y) \in \RR^D$ we have
    \begin{equation}
        \label{eq:lem:laplacian-rho-t-off-manifold}
        \begin{split}
            \Delta \rho_t(x, y) & = \left(\frac{\norm{y}^2}{t^2} + \frac{d-D}{t}\right)\rho_t(x, y)\\
            & + \phi_t^{D-d}(y) \Delta_x(\psi * \phi_t^d)(x).            
        \end{split}
    \end{equation}
\end{lemma}

\begin{proof}
    By using the product decomposition~\eqref{eq:rho_t}, we get
    \begin{equation}
        \begin{split}
            \Delta &\rho_t(x,y) =\\
            &\psi * \phi_t^d(x) \Delta_y\phi_t^{D-d}(y) + \Delta_x (\psi * \phi_t^d)(x) \phi_t^{D-d}(y)    
        \end{split}
    \end{equation}
    To derive the first term, we note that a direct computation yields
    \begin{equation}
    \label{eq:laplacian_gaussian}
        \Delta_y\phi_t^{D-d}(y) = \phi_t^{D-d}(y) \left( \frac{\norm{y}^2}{t^2} + \frac{d-D}{t} \right). \qedhere
    \end{equation}
\end{proof}

As a consequence, by putting $y=0$ and using $\phi_t^{D-d}(0) = (2\pi t)^{(d-D)/2}$ we obtain the following. 

\begin{corollary}
    \label{cor:laplacian-rho-t-on-manifold}
    For $t>0$ and $x\in \RR^d$ we have
    \begin{equation}
        \Delta \rho_t(x, 0) = \frac{d-D}{t}\rho_t(x, 0) + (2\pi t)^{(d-D)/2} \Delta_x(\psi * \phi_t^d)(x). 
    \end{equation}
\end{corollary}

Let us stop here for a moment to discuss the expression $\Delta_x (\psi * \phi_t^d)$. When applying a differential operator to a convolution, under certain regularity conditions we can move it to either of the factors, 
which can turn out to be quite helpful. As it will turn out in the examples, $\Delta_x \psi * \phi_t^d$ will be the most desired form of $\Delta_x (\psi * \phi_t^d)$.

More precisely, given a convolution $f * g$, and a $k$-th order differential operator $L$, the following conditions guarantee that $L(f*g) = Lf * g$ (conditions for the other equality follow from commutativity of convolution):
\begin{enumerate}
    \item All partial derivatives of $f$ up to order $k$ (including $0$-th order, i.e.\ $f$ itself) exist and are bounded.
    \item The function $g$ is integrable, i.e. $\int \lvert{g}\rvert < \infty$.
\end{enumerate}
The second condition is clearly satisfied by densities of probability distributions, and $\phi^d_t$ satisfies the first condition. Hence, for any $\psi$ we can  write 
\begin{equation}
    \Delta_x (\psi * \phi_t^d) =  \psi * \Delta_x\phi_t^d.
\end{equation}
The second equality however requires additional assumptions;
\begin{equation}
    \label{eq:conv-first-factor}
    \Delta_x (\psi * \phi_t^d) =  \Delta_x\psi * \phi_t^d,
\end{equation}
holds if $\psi$ is bounded and has bounded partial derivatives up to order $2$.

\section{Equivalent formulations of LIDL}
\label{app:equivalent-formulations-of-lidl}
The following proposition allows us to conclude that the two approaches to LID estimation we discussed are equivalent.

\begin{proposition}
    \label{prop:equivalent-asymptotics}
    Given a strictly positive differentiable function $f\colon (0,a) \to (0,\infty)$ and a positive real number $\alpha > 0$, the following conditions are equivalent.
    \begin{enumerate}
        \item The function $f$ explodes at $0$ like $t^{-\alpha}$, i.e.\ for some positive constants $c,C > 0$ one has $c < t^{\alpha}f(t) < C$ for some $\epsilon>0$ and $t\in (0,\epsilon)$.
        \item $\log f(t) = -\alpha\log t + O(1)$.
        \item $\lim_{t\to 0^+} {\log f(t)}/{\log t} = -\alpha$.
        \item $\lim_{t\to 0^+} {tf'(t)}/{f(t)} = -\alpha$.
    \end{enumerate}
\end{proposition}

\begin{proof}
    All limits will be taken with $t\to 0$, and we will omit this from notation. 
    
    If $t < t^\alpha f(t) < C$, then $\log (t^\alpha f(t)) = O(1)$. Rearranging leads to point 2. Going further, dividing by $\log x$ and observing that $\lim O(1)/\log t = 0$ yields point 3.

    Now, assume condition 3. Since $\lim\log t = -\infty$, for the whole expression to tend to $-\alpha < 0$, we need $\lim\log f(t) = \infty$. In this case however we may apply the de l'Hospital rule giving rise to point 4.

    It remains to show the final implication. Let $g(t) = t^\alpha f(t) > 0$ be the function we want to bound. Substituting $f(t) = t^{-\alpha}g(t)$ into $\lim_{t\to 0} {tf'(t)}/{f(t)} = -\alpha$, we obtain
    \begin{equation}
        \begin{split}
                \lim \frac{t (-\alpha t^{-\alpha-1}g(t) + t^{-\alpha}g'(t))}{t^{-\alpha} g(t)} &= \\
                &= \lim \left( -\alpha + \frac{g'(t)}{g(t)} \right) \\
                &= -\alpha,
        \end{split}
    \end{equation}
    amounting to $\lim (\log g(t))' = 0$. But this means that for some $\epsilon>0$, the derivative $(\log g(t))'$ is bounded on $(0,\epsilon)$, implying that $\log g(t)$ is Lipschitz and therefore bounded on this interval. This yields the desired estimates.
\end{proof}

Under the notation of Section~\ref{sec:reformulating-lidl} we have the following two propositions.

\begin{proposition}
    \label{prop:app:equivalent-lidl}
    For $t$ near $0$ the following estimate holds
    \begin{equation}
        \log\rho_t(x) = \beta(x)\log{\sqrt{t}} + O(1).
    \end{equation}  
\end{proposition}
\begin{proof}
    By using Proposition~\ref{prop:equivalent-asymptotics} with $f(t) = \rho_t(x)$, we can see that its last condition holds with $\alpha = -\beta(x)/2$. Thus, the equivalent condition 2.\ yields the desired equality.
\end{proof}

The next proposition provides an elegant expression for $\beta_t(x)$, and consequently for $\beta(x)$, expressed in terms of the density $\psi$ on $\RR^d$.

\begin{proposition}
\label{prop:app:beta_final}
    For $t>0$ and $x\in S = \RR^d \subseteq \RR^D$ we have
    \begin{equation}
        \beta_t(x) = d-D +  \frac{\Delta_x(\psi * \phi_t^d)(x)}{\psi * \phi_t^d(x)} \cdot t.
    \end{equation}
\end{proposition}
\begin{proof}
    By \eqref{eq:def-beta-t} and Lemma~\ref{lem:laplacian-rho-t-everywhere} we obtain
    \begin{equation}
        \begin{split}
            \beta_t(x) &= \frac{t}{\rho_t(x)} \Delta \rho_t(x) \\
        &= d-D + \frac{t(2\pi t)^{(d-D)/2} \Delta_x(\psi * \phi_t^d)(x)}{\rho_t(x)},
        \end{split}
    \end{equation}
    where in the second term $(2\pi t)^{(d-D)/2}$ cancels out after expanding $\rho_t$ in the denominator.
\end{proof}

\section{Convex combinations of distributions}
\label{app:convex-combinations}

Our new approach allows to easily analyze the behavior of LIDL on unions of manifolds. Suppose that $S = \bigcup_i S_i$ is a finite union of manifolds. The measure $p_S$ can be then decomposed into a convex combination of probability measures $p_i$ supported on $S_i$, namely 
\begin{equation}
    \label{eq:convex-decomposition}
    p_S = \sum_i \lambda_i p_i,
\end{equation}
where $\lambda_i > 0$ and $\sum_i \lambda_i = 1$. With each $(S_i, p_i)$ we may associate their corresponding diffused density $\rho_t^i$, reparameterized slope $\beta_t^i$, and its limit $\beta^i$ through equations~\eqref{eq:diffusion-solution-general}, \eqref{eq:def-beta-t}, and~\eqref{eq:def-beta}. Due to bilinearity of convolution, analogous decomposition holds for $\rho_t$, namely
\begin{equation}
    \rho_t(x) = \sum_i \lambda_i \rho_t^i(x).
\end{equation}
Using these, we may now attempt to reduce the problem to the study of individual components.

In the following discussion it turns out that the decomposition of $S$ into a union of manifolds, and the fact that $p_i$ are supported on the components of this decomposition, are of no importance. The only essential assumption is the convex combination decomposition~\eqref{eq:convex-decomposition}. 

\begin{lemma}
    \label{lem:convex-combination}
    A convex combination decomposition $p_S=\sum_i\lambda_i p_i$, gives rise to a decomposition of $\beta_t(x)$ into a convex combination of $\beta_t^i(x)$,
    \begin{equation}
        \beta_t(x) = \sum_i \frac{\lambda_i \rho_t^i(x)}{\rho_t(x)} \beta_t^i(x).
    \end{equation}
    In particular, if all the limits below exist, one has
    \begin{equation}
        \beta(x) = \sum_i \left(\lim_{t\to 0^+} \frac{\lambda_i \rho_t^i(x)}{\rho_t(x)}\right) \beta^i(x).
    \end{equation}
\end{lemma}
\begin{proof}
    The desired decomposition can be obtained through direct expansion of the definition of $\beta_t$. We have
    \begin{equation}
        \begin{split}
        \beta_t(x) 
        & = \sum_i \frac{t\lambda_i\Delta\rho_t^i(x)}{\rho_t(x)} \\
        & = \sum_i \frac{\lambda_i\rho_t^i(x)}{\rho_t(x)} \cdot \frac{t\Delta \rho_t^i(x)}{\rho_t^i(x)} \\
        & = \sum_i \frac{\lambda_i \rho_t^i(x)}{\rho_t(x)} \beta_t^i(x).
    \end{split}
    \end{equation}
    The second assertion follows by taking the limit.
\end{proof}

Now take a closer look at the coefficients $\lambda_i \rho_t^i(x) / \rho_t(x)$.

\begin{lemma}
    \label{lem:convex-combination-coeff-estimate}
    Suppose that $p_S=\sum_i\lambda_i p_i$ is a convex combination decomposition, and for some $R>r>0$ there exist $i\ne j$ such that $p_i(B(x,R)) = 0$ and $p_j(B(x,r)) = C > 0$. Then the following estimate holds.
    \begin{equation}
        \frac{\lambda_i\rho_t^i(x)}{\rho_t(x)} \leq \frac{\lambda_i}{\lambda_i+ C\lambda_j e^{(R^2-r^2)/2t}}.
    \end{equation}
    Moreover, 
    \begin{equation}
        \lim_{t\to 0^+} \frac{\lambda_i\rho_t^i(x)}{\rho_t(x)} = 0,
    \end{equation}
    and $p_i$ does not contribute to $\beta(x)$.
\end{lemma}
\begin{proof}
    First, observe that the integral defining $\rho_t^i(x)$ can be estimated by the supremum of the integrand $\phi_t^D(x-y)$ over the support of $p_i$, contained in $B(x,R)^c$, namely
    \begin{equation}
    \label{eq:convex_comb_of_beta}
    \begin{split}
        \rho_t^i(x) &= \int \phi_t^D(x-y)\,dp_i(y) \\
        & \leq \sup_{y\in B(x,R)^c} \phi_t^D(x-y) = (2\pi t)^{-D/2} e^{-R^2/2t}.
    \end{split}
    \end{equation}
    Therefore, the ratio $\rho_t^j(x)/\rho_t^i(x)$ can be bounded from below as
    \begin{equation}
        \begin{split}
            \frac{\rho_t^j(x)}{\rho_t^i(x)} & 
            \geq \int e^{(R^2 - \norm{x-y}^2)/2t} \,dp_j(y) \\
            & \geq \int_{B(x,r)} e^{(R^2 - \norm{x-y}^2)/2t} \,dp_j(y) \\
            & \geq \int_{B(x,r)} e^{(R^2 - r^2)/2t} \,dp_j(y) = Ce^{(R^2 - r^2)/2t}.
        \end{split}
    \end{equation}
    Combining both these estimates, we end up with
    \begin{equation}
        \begin{split}
            \frac{\rho_t(x)}{\lambda_i\rho_t^i(x)}
            & = 1 + \sum_{k\ne i} \frac{\lambda_k\rho_t^k(x)}{\lambda_i\rho_t^i(x)}\\ & \geq  1 + \frac{\lambda_j\rho_t^j(x)}{\lambda_i\rho_t^i(x)} \geq 1 + \frac{\lambda_j}{\lambda_i} Ce^{(R^2 - r^2)/2t}.
        \end{split}
    \end{equation}
    Inverting and taking the limit yields the desired assertions.
\end{proof}

\end{document}